\documentclass{article}
\usepackage{ijcai19}

\usepackage{times}
\usepackage{soul}
\usepackage{url}
\usepackage[hidelinks]{hyperref}
\usepackage[utf8]{inputenc}
\usepackage[small]{caption}
\usepackage{graphicx}
\usepackage{import}
\newcommand\independent{\protect\mathpalette{\protect\independenT}{\perp}}
\def\independenT#1#2{\mathrel{\rlap{ $#1#2$ }\mkern2mu{#1#2}}}

\usepackage{xcolor}
\usepackage{fncylab}
\labelformat{equation}{(#1)} 
\usepackage{subfigure}

\usepackage{amsmath}
\usepackage{booktabs}
\usepackage{algorithm}
\usepackage{amsfonts, amssymb}
\usepackage{algorithmic}
\usepackage{bm}
\hypersetup{draft}

\usepackage{mathtools}
\usepackage{amsthm}

\newtheorem{theorem}{Theorem}

\title{Learning Disentangled Semantic Representation for Domain Adaptation}

\iftrue
\author{
Ruichu Cai$^1$
\and
Zijian Li$^1$\and
Pengfei Wei$^{2}$\and
Jie Qiao$^{1}$\and
Kun Zhang$^{3}$ \And
Zhifeng Hao$^{4}$
\affiliations
$^1$School of Computers, Guangdong University of Technology, China\\
$^2$School of Computer Science and Engineering, Nanyang Technological University, Singapore\\
$^3$Department of Philosophy, Carnegie Mellon University, USA\\
$^4$School of Mathematics and Big Data, Foshan University, China\\
\emails
cairuichu@gdut.edu.cn,
leizigin@gmail.com,
wpf89928@gmail.com,
kunz1@cmu.edu,
qiaojie.chn@gmail.com,
zfhao@gdut.edu.cn
}
\fi

\begin{document}

\maketitle

\begin{abstract}

Domain adaptation is an important but challenging task. Most of the existing domain adaptation methods struggle to extract the domain-invariant representation on the feature space with entangling domain information and semantic information. Different from previous efforts on the entangled feature space, we aim to extract the domain invariant semantic information in the latent disentangled semantic representation (DSR) of the data. In DSR, we assume the data generation process is controlled by two independent sets of variables, i.e., the semantic latent variables and the domain latent variables. Under the above assumption, we employ a variational auto-encoder to reconstruct the semantic latent variables and domain latent variables behind the data. We further devise a dual adversarial network to disentangle these two sets of reconstructed latent variables. The disentangled semantic latent variables are finally adapted across the domains. Experimental studies testify that our model yields state-of-the-art performance on several domain adaptation benchmark datasets.

\end{abstract}

\section{Introduction}
Domain adaptation is an important but challenging task. Since the acquisition of a large labeled data is usually either expensive or impractical, how to train on the unlabeled target domain with the help of labeled source domain has become a particular focus. However, this learning scheme suffers from a well-known phenomenon named \textit{domain shift}, leading an urging motivation in building an adaptive classifier that can efficiently transfer the source labeled data under the domain shift, this problem is also known as \textit{unsupervised domain adaptation}.

\begin{figure*}[htbp]
	\centering
	\subfigure[Latent manifold of the data generation.]{
		\includegraphics[width=0.2993\textwidth]{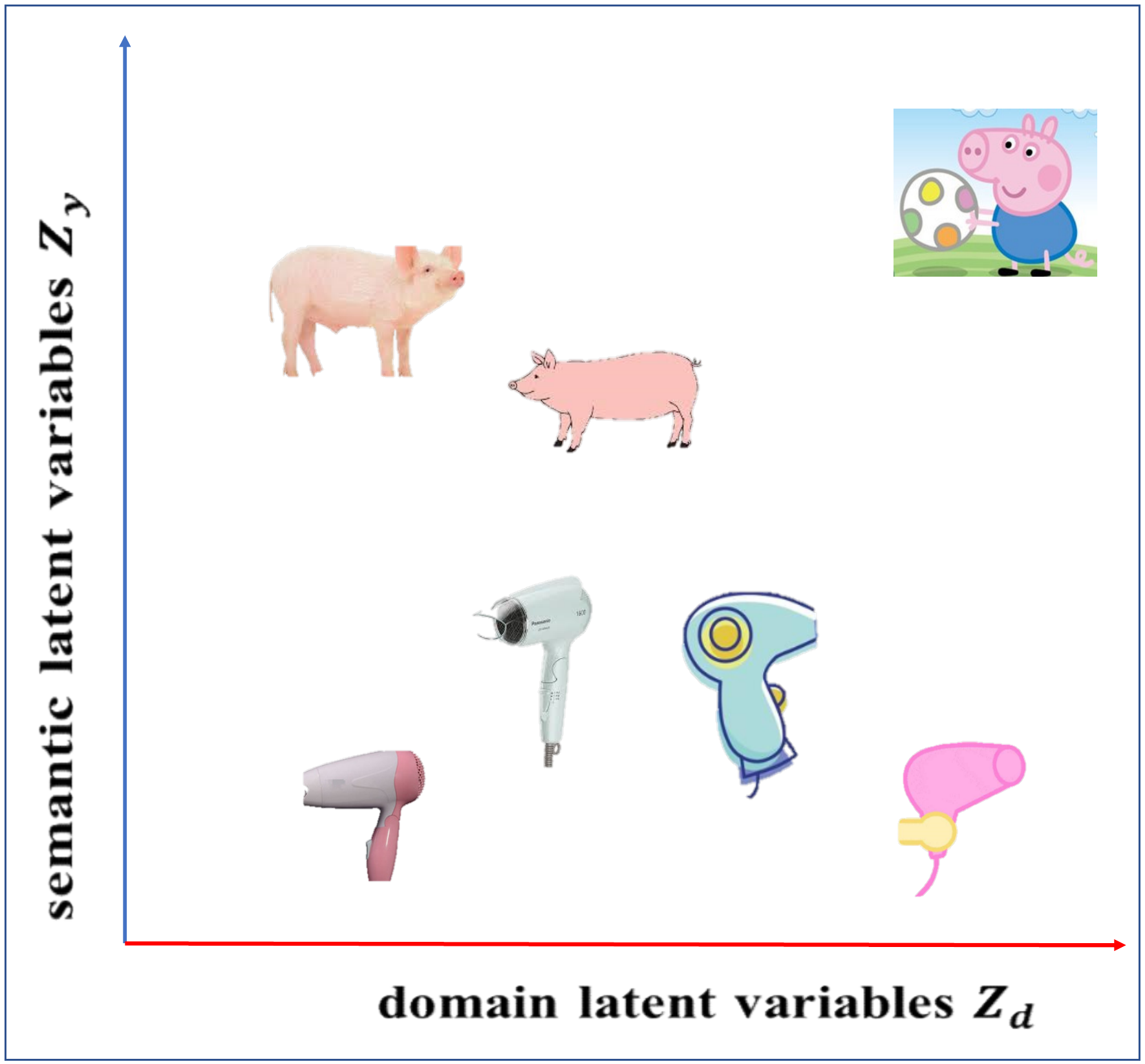}
		\label{fig:intro11}
	}
	\subfigure[Distorted feature manifold with residual domain information]{
		\includegraphics[width=0.3\textwidth]{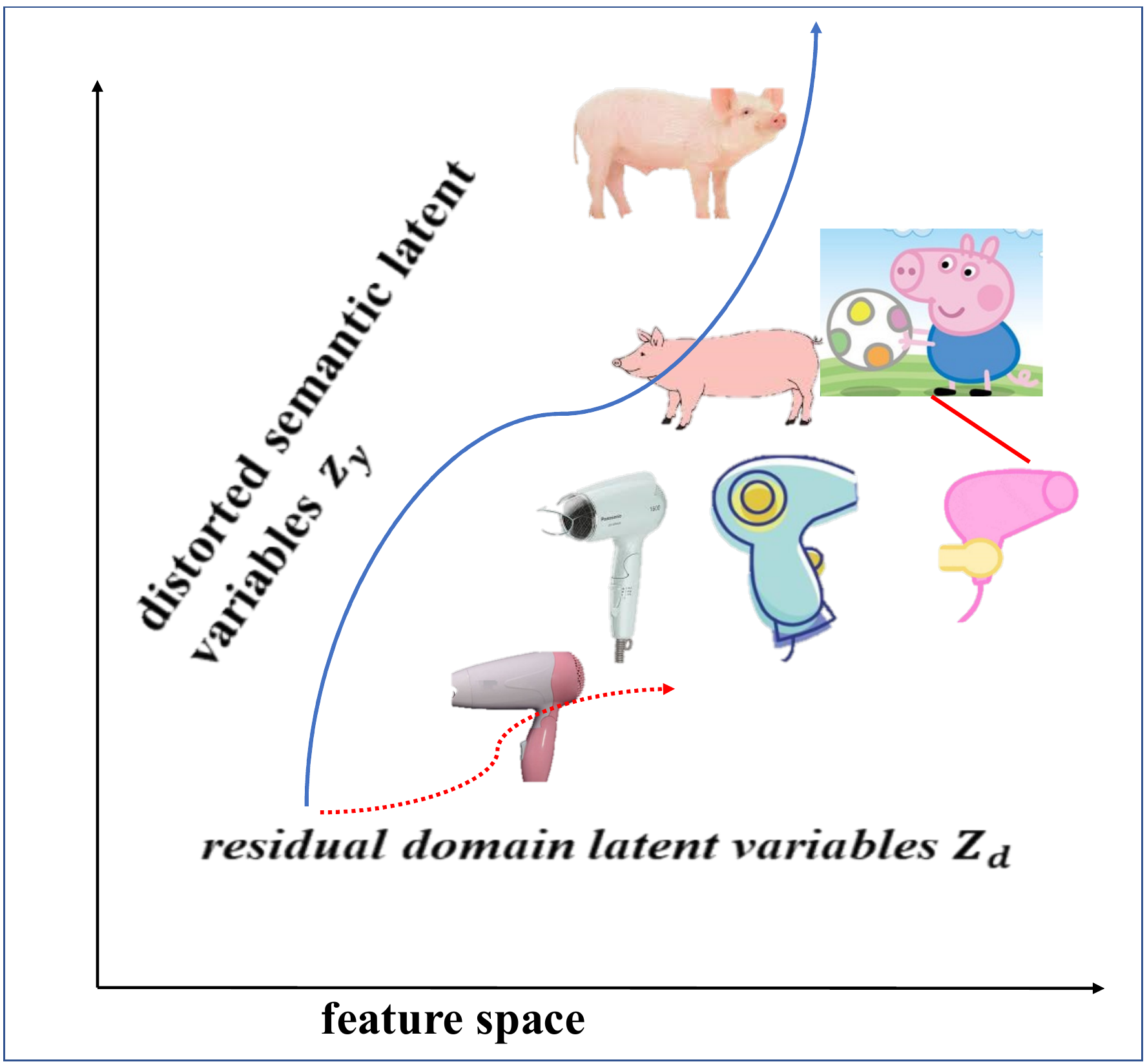}
		\label{fig:intro12}
	}
	\subfigure[Disentangle semantic manifold]{
		\includegraphics[width=0.3052\textwidth]{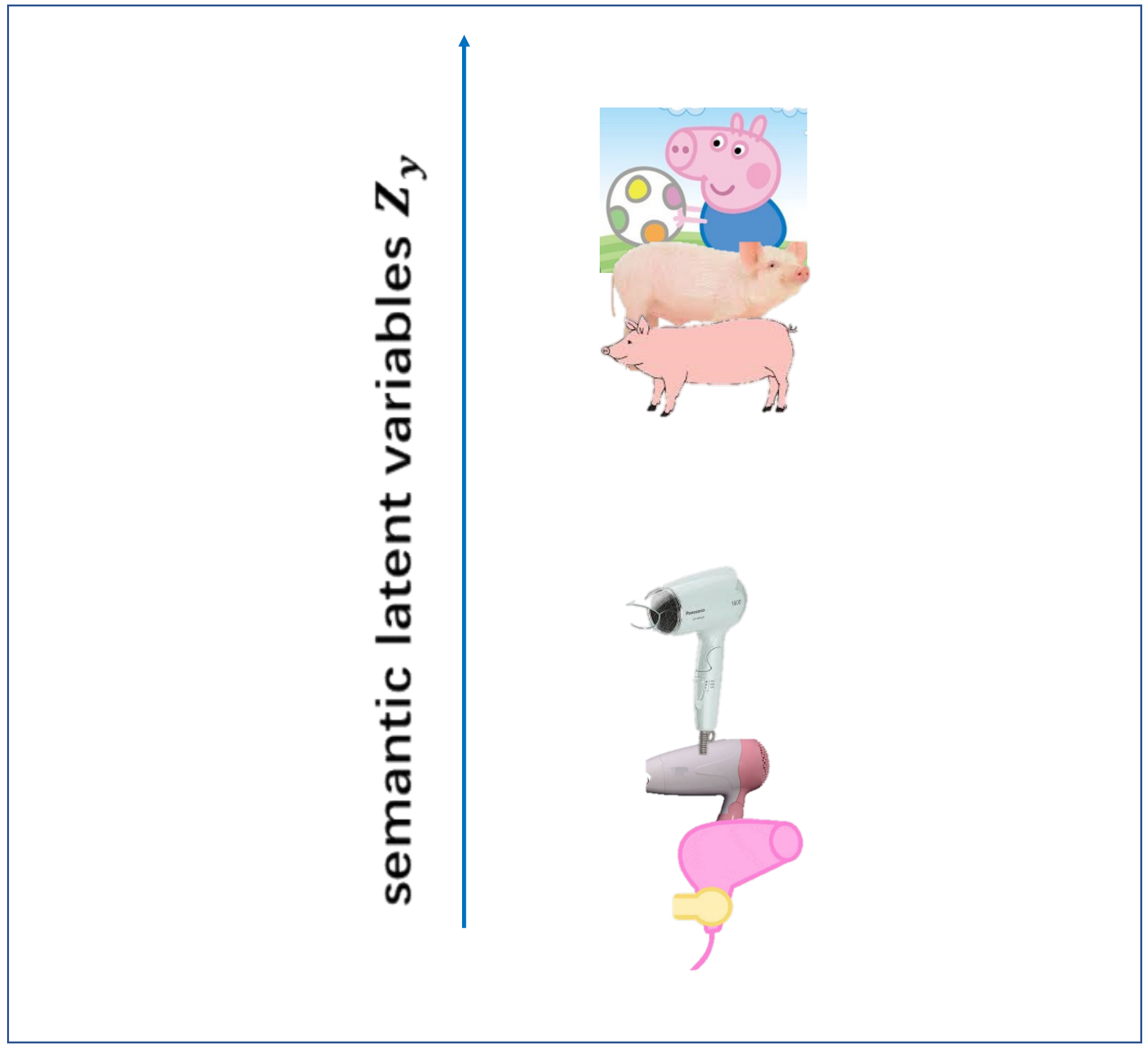}
		\label{fig:intro13}
	}
	\caption{A toy domain adaption example with ``pig'' and ``hair drier'' samples on the ``art'' and ``camera'' domain. (a) The data generation process is controlled by the disentangled domain latent variables and semantic latent variables in the latent manifold. (b) The samples are distorted on the feature manifold, and the residual domain information results in the false alignment between the Peppa Pig and the hair drier. (c) The samples are well distributed on the disentangled semantic manifold. ($\textit{best view in color}$)}
	\label{fig:intro}
\end{figure*}

An essential approach in unsupervised domain adaptation is to understand what the domain-invariant representation across the domains is and how to find it \cite{zhang2013domain,pan2011domain,gong2012geodesic}. Typical methodologies explored in the literature include the feature alignment approaches that extract domain-invariant representation by minimizing the discrepancy between the feature distributions inside deep feed-forward architectures \cite{tzeng2014deep,long2015learning,long2016unsupervised,long2017deep}, and the adversarial learning approaches that extract the representation by deceiving the domain discriminators \cite{tzeng2015simultaneous,ganin2015unsupervised,ganin2016domain,long2018conditional}. Recently, a fine-grained semantic alignment has also been proposed in order to extract domain-invariant representation under the consideration of the semantic information \cite{xie2018learning,zhang2018domain,chen2018progressive}. However, most of them require target pseudo labels in order to minimize the discrepancies across domains within the same labels, and thus resulting the error accumulation due to the uncertainty of the pseudo-labeling accuracy. 

Due to the complex manifold structures underlying the data distributions, these methods mainly suffer a false alignment problem \cite{pei2018multi,xie2018learning}. As shown in Figure \ref{fig:intro11}, the data generation process is controlled by the disentangled domain latent variables and semantic latent variables in the latent manifold. The ideal \textit{semantic} position of the two types of labels -- pig and hair drier -- should be placed relatively upper and lower on the manifold according to the semantic axis, and the ideal \textit{domain} position within the same labels should be placed in the left and right according to the domain axis. However, as shown in Figure \ref{fig:intro12}, once the domain information is not completely removed, samples are distorted on the feature manifold, leading to the false alignment problem, e.g., the Peppa Pig looks like a pink hair-drier and thus the feature of the Peppa Pig might be near to that of the hair-drier in the distorted feature manifold.  


Motivated by the example of Figure \ref{fig:intro12}, the underlying cause of the false alignment problem is the entanglement of the semantic and domain information.  More specifically, samples are controlled by two sets of independent latent variables $\mathbf{z}_y$ and $\mathbf{z}_d$. However, these two sets of latent variables are highly tangled and distorted on the high dimensional feature manifold space. It is very challenging to remove the domain information while preserving the semantic information on a complex tangled feature manifold space. 

In this work, motivated by the disentanglement property of the multiple explanatory factors in the representation learning literature \cite{bengio2013representation,dinh2014nice}, we propose a Disentangle Semantic Representation learning model (DSR in short) by assuming the independence between the semantic variables $\bm{z}_y$ and the domain variables $\bm{z}_d$. Our DSR reconstructs the disentangled latent space and simultaneously uses the semantic variables to predict the target labels. The underlying intuition, as shown in Figure \ref{fig:intro13}, is that by using the disentangle semantic latent variables that are independent to the domain latent variables, we can easily classify the labels into two categories merely based on the semantic axis $\mathbf{z}_y$. We employ a variational auto-encoder to reconstruct the disentangled semantic and domain latent variables, with the help of a dual adversarial network. The extensive experimental studies demonstrate that DSR outperforms state-of-the-art unsupervised domain adaptation methods on standard domain adaptation benchmarks.

%

\section{Related Work}

Deep feature learning methods have been shown very effective for unsupervised domain adaptation. The key idea of the deep feature learning is to extract domain-invariant representation by aligning different domains. Some works utilize maximum mean discrepancy(MMD) to realize the domain alignment. \cite{tzeng2014deep} learns domain-invariant representation by adding an adaptation layer and an additional domain confusion loss; \cite{long2015learning} reduces the domain discrepancy by using an optimal multi-kernel selection method. \cite{long2016unsupervised} assumes that the source classifier and target classifier differ by a residual function and enable classifier adaptation by plugging several layers with reference to the target classifier. \cite{long2017deep} learns domain-invariant representation by aligning the joint distributions of multiple domain-specific layers across domains based on a joint maximum mean discrepancy criterion. 
Other works introduce a domain adversial layer for the domain alignment. \cite{ganin2015unsupervised} introduces a gradient reversal layer to fool the domain classifier and extracts the domain-invariant representation, \cite{tzeng2017adversarial} borrows the idea of generative adversarial network(GAN)\cite{goodfellow2014generative} and proposes a novel unified framework for adversarial domain adaptation. 

Recent studies also show the benefits of semantic alignment to unsupervised domain adaptation. With the assumption that the distance among the samples with the same label but from different domains should be as small as possible, \cite{xie2018learning} learns semantic representation by aligning labeled source centroid and pseudo-labeled target centroid. \cite{chen2018progressive} aligns the discriminative features across domains progressively and effectively, via utilizing the intra-class variation in the target domain. \cite{deng2018domain} proposes a similarity constrained alignment method which enforces a similarity-preserving constraint to maintain class-level relations among the source and target samples.

However, most of the semantic alignment methods require target pseudo labels to minimize the variety discrepancies across domains, and thus resulting the error accumulation due to the uncertainty of the pseudo-labeling accuracy. In this work, we employ the concept of variational auto-encoder \cite{kingma2013auto} and adversarial learning to extract the domain invariant semantic representation. 

\section{Disentangled Semantic Representation Model}
In this work, we focus on the unsupervised domain adaptation problem that uses the labeled samples $ D_{S}=\left\{\bm{x}_{i}^{S}, y_{i}^{S}\right\}_{i=1}^{n_{S}} $ on the source domain to classify the unlabeled samples $ D_{T}=\left\{\bm{x}_{j}^{T}\right\}_{j=1}^{n_{T}} $ on the target domain. The goal of this paper is to understand:$(1)$ what the domain-invariant representation across domains is, and $(2)$ how to design a framework that can extract such a domain-invariant representation. 

\begin{figure}[htbp]
    \centering
    \includegraphics[width=0.3\columnwidth]{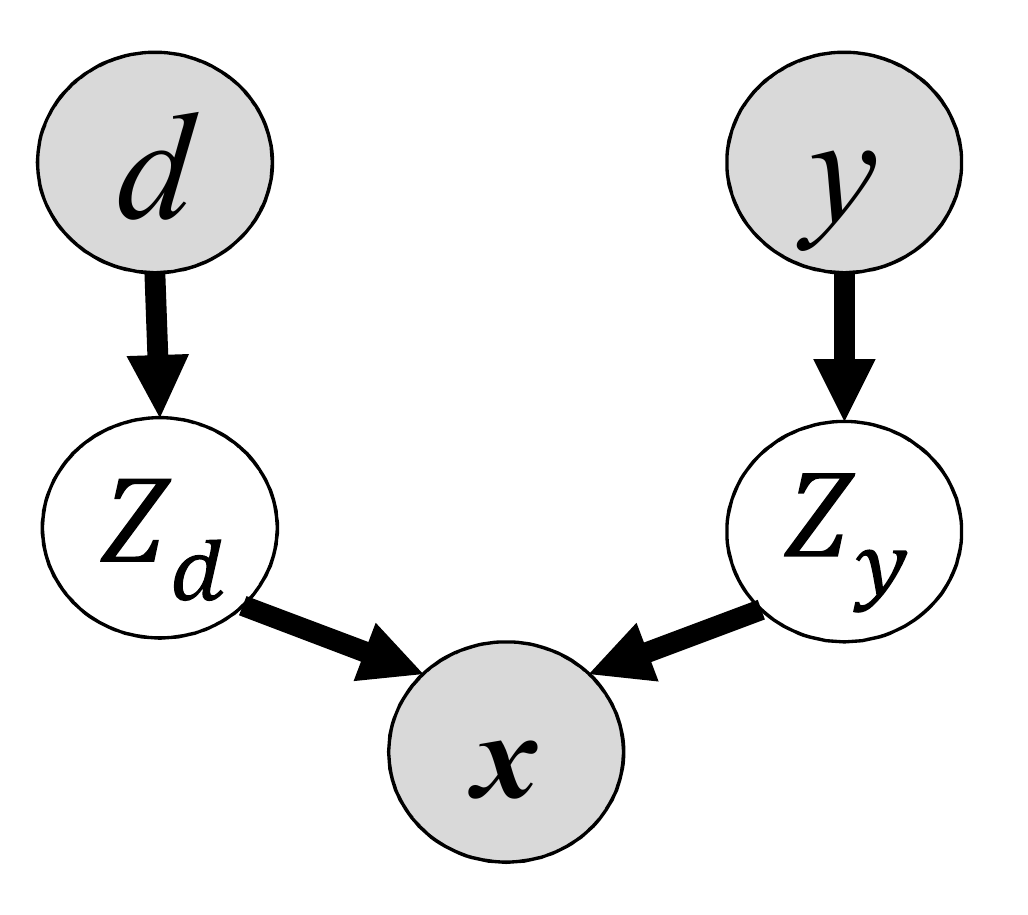}
    \caption{The causal model of data generation process, which are controlled by the latent variables $\bm{z}_d$ and $\bm{z}_y$.}
    \label{fig:generation}
\end{figure}
Regarding the first point, we start from the causal mechanism behind the data generation process as shown in Figure. \ref{fig:generation}. Given $\bm{x}$, it is generated from two independent latent variables, .i.e, $\bm{z}_d$ encodes the domain information and $\bm{z}_y$ encodes the semantic information. $\mathbf{z}_d\in \mathbb{R}^{K_d}$ and $\mathbf{z}_y\in \mathbb{R}^{K_y}$ denote the semantic latent variables and the domain latent variables respectively. Considering that the domain information may differ considerably across domains, we induce that the semantic latent variables play an important role in extracting the domain-invariant representation. Let $\bm{z}=\{\mathbf{z}_y, \mathbf{z}_d \}$. By further developing the independence property in the latent space, we also assume that $\mathbf{z}_{y} \independent \mathbf{z}_{d}$. 

Regarding the second point, with the above data generative mechanism, we propose a disentangled semantic representation (DSR) domain adaptation framework by first reconstructing the two independent latent variables via variational auto-encoder and then disentangling them through a dual adversarial training network. The key structure of the proposed framework is given in Figure. \ref{fig:model}. 

\definecolor{a}{RGB}{0,176,240}
\definecolor{b}{RGB}{146, 208, 80}

As shown in the upper part of Figure. \ref{fig:model}, the ``Reconstruction'' architecture, we first obtain a feature 
via a backbone feature extractor $G(.) $ .e.g ResNet, and then use the VAE-like scheme to reconstruct the feature by first encoding it into the latent variables $\bm{z}$, then use the latent variables to reconstruct the feature $G(\bm{x})$ from two independent latent variables $\bm{z}_y$ and $\bm{z}_d$. However, unlike the vanilla VAE, we further design a ``Disentanglement'' architecture as shown in the green dot box of Figure. \ref{fig:model}. In this architecture, two adversarial modules are placed under the semantic latent variables and domain latent variables respectively. For the label adversarial learning module on the left side, it aims to pull all the semantic information into $\bm{z}_y$ and push all the domain information from $\bm{z}_d$. For the domain adversarial learning module on the right side, it aims to pull all the domain information into $\bm{z}_d$ and push all the semantic information from $\bm{z}_y$. By doing so, we can obtain those domain-invariant semantic information without the contamination of the domain information.

We introduce more technical details of our proposed framework in the following section.

\begin{figure}[htbp] 
	\centering
	\includegraphics[width=0.48\textwidth]{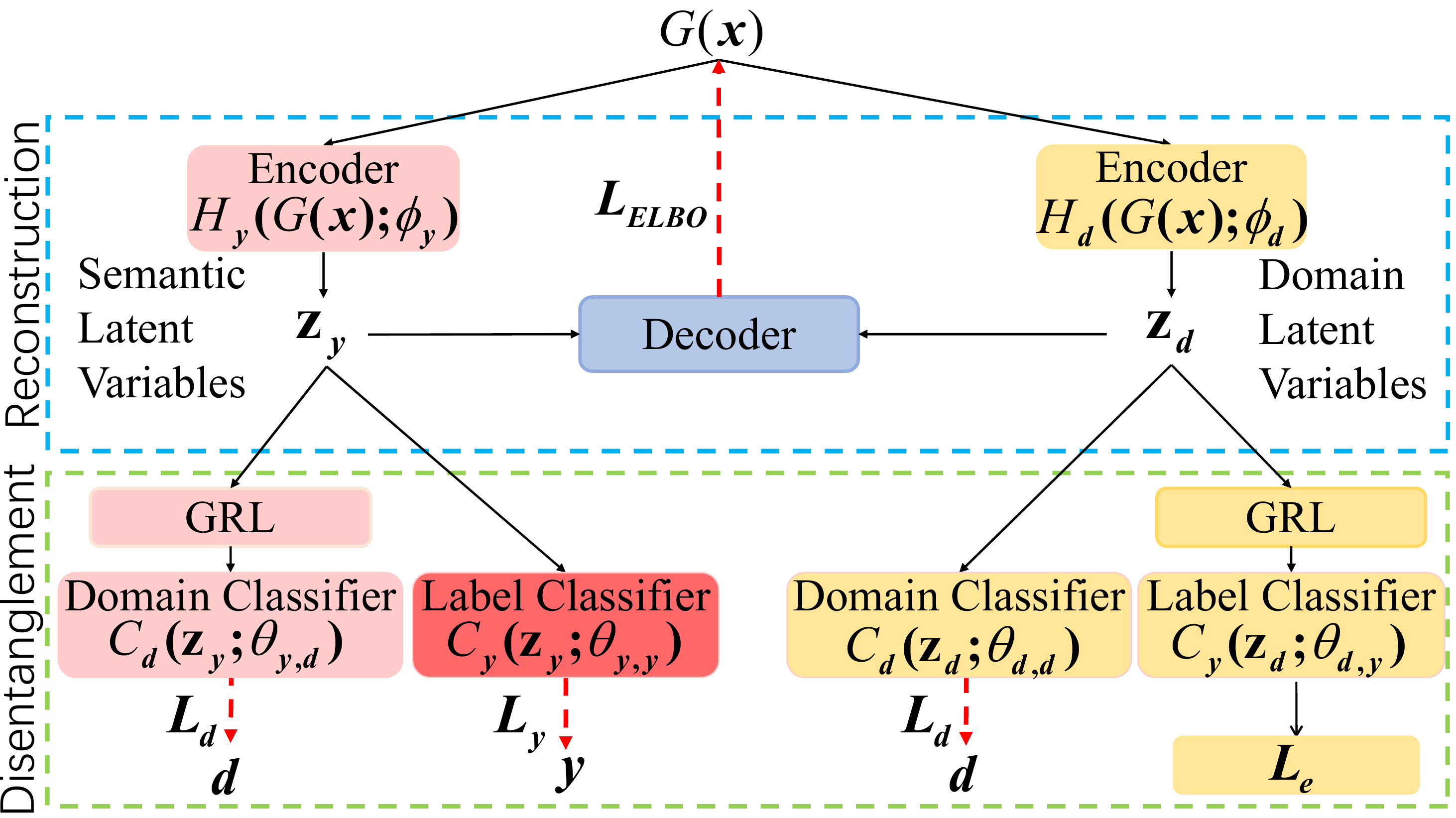}
	\caption{The framework of the Disentangled Semantic Representation model. In the reconstruction block (marked with the blue dashed lines), the variational auto-encoder is used to recover the semantic latent variables ($\bm{z}_y$) and the domain latent variables ($\bm{z}_d$). In the disentanglement block (marked with the green dashed lines), a dual adversarial network is used to disentangle the latent variables. $H_y$ and $H_d$ are the encoders for the semantic and domain information respectively. $C_y$ and $C_d$ are the classifiers for the label and domain respectively. GRL is a gradient reversal layer that multiplies the gradient by a negative constant. ($\textit{best view in color}$)}
	\label{fig:model}
\end{figure}

\subsection{Semantic Latent Variables Reconstruction}
For the reconstruction architecture in DSR framework, we follow the configuration in VAE. We denote $q_{\phi}(\bm{z} | \bm{x}) $ as the \textit{encoder} with respect to $\bm{\phi}$ to approximate the intractable true posterior $p(\bm{z} | \bm{x})$. The variational lower bound of the marginal likelihood is given as follow
\begin{multline}
 \mathcal{L}_{\mathrm{ELBO}}(\bm{\phi}, \theta_r)=\\-D_{KL}\left(q_{\bm{\phi}}(\bm{z} | \bm{x})| | P(\bm{z})\right)+ \mathbb{E}_{q_{\bm{\phi}}(\bm{z} | \bm{x})}\left[\log P_{\theta_r}(\bm{x} | \bm{z})\right] ,
 \label{eq:elob1}
\end{multline}
where $ P_{\theta_r}(\bm{x} | \bm{z})$ denotes the \textit{decoder} with respect to the parameters $\theta_r$ and $P(\bm{z})$ is the prior distribution. Then, we further decompose the latent variables $\bm{z}$ into $\bm{z_y}$ and $\bm{z_d}$. Eq. \ref{eq:elob1} can be derived as follows:
\begin{multline}
\begin{aligned}
\mathcal{L}_{\mathrm{ELBO}} (\phi _{y} ,\phi _{d} ,\theta _{r} )=-D_{KL}( q_{\phi _{y}}(\bm{z}_{y} |G(\bm{x} )) \| P(\bm{z}_{y}))\\
-D_{KL}( q_{\phi _{d}}(\bm{z}_{d} |G(\bm{x} )) \| P(\bm{z}_{d}))\\
+\mathbb{E}_{q_{\phi _{y,d}} (\bm{z}_{y} ,\bm{z}_{d} |G(\bm{x} ))}[\log P_{\theta _{r}} (G(\bm{x} )|\bm{z}_{y} ,\bm{z}_{d} )] .
\end{aligned}
\end{multline}
Here, we assume that $P(\bm{z}_{y}),P(\bm{z}_{d}) \sim \mathcal{N} (\bm{0},{\mathbf{I}})$, and $ \phi_y $ and $\phi_d$ are the parameters of the encoder. Similar to VAE, by applying a reparameterization trick, we use MLP $H_y( G(\mathbf{x}) ;\phi_y )$ and $H_d( G(\mathbf{x}) ;\phi_d )$ as the universal approximator of $q$ to encode the data into $\bm{z}_y$ and $\bm{z}_d$ respectively.

\subsection{Semantic Latent Variables Disentanglement}

For the disentanglement architecture of the DSR framework, it consists of two adversarial modules working together following the typical configuration in \cite{ganin2015unsupervised}. On the left of the Figure. \ref{fig:model} under the semantic latent variables $\bm{z}_y$ is the \textit{label adversarial learning module} that fuses the semantic information and excludes all the domain information. This is done by using a label classifier $C_y$ and a domain classifier $C_d$. To exclude the domain information, we use a gradient reversal layer (GRL) for $C_d$. As a result, the parameters $ \phi_y$ in $H_y$ are learned by maximizing the loss $L_d$ of $C_d$ and simultaneously minimizing the loss $L_y$ of $C_y$. The parameters $ \theta_{y,d}$ of $ C_d $ are learned by $L_d$. The overall objective function of label adversarial learning module is shown as follows:
\begin{multline}
\begin{aligned}
\mathcal{L}_{sem}&(\phi_y, \theta_{y,y}, \theta_{y,d})=\\ &\frac{\delta}{n_{S}} \sum_{x_{i}^{s} \in D_{S}} L_{y}\left(C_y\left(H_{y}\left(G(\bm{x});\phi_y\right);\theta_{y,y}\right), y_{i}\right) \\-& \frac{\lambda}{n}\sum_{x_{i} \in\left(D_{S}, D_{T}\right)} L_{d}\left(C_d\left(H_{y}\left(G(\bm{x});\phi_y\right);\theta_{y,d}\right), d_{i}\right), 
\end{aligned}
\end{multline}

where $ n=n_{S}+n_{T} $, $ \lambda $ are a trade-off parameters that balance the two objectives and $\delta$ is the parameter that controls 
the weight of $L_y$. A bigger $\delta$ enables the label classifier to learn more semantic information. The default value of $\delta$ is 1, and we try different values in order to validate the individual contributions of the domain adversarial learning module in the section 4.

Similarly, on the right-side adversarial module is the \textit{domain adversarial learning module} that fuses the domain information to $\bm{z}_d$ and excludes the semantic information from $\bm{z}_d$. The GRL is placed on the label classifier above in order to absorb all the domain information from $\bm{z}_y$. However, unlike the semantic module, we do not use cross-entropy as the label loss, because of the unsupervised learning in the target domain. In order to utilize the data in the target domain, we employ maximum entropy loss $L_e$ for the label classifier $C_y$. As a result, the parameters $ \phi_d$ in $H_d$ are learned by maximizing the loss $L_e$ of label classifier $C_y$ but minimizing the loss $L_d$ of label classifier $C_d$. In addition, the parameters $ \theta_{d,y}$ of label classifier $ C_y $ are learned by minimizing its own loss $L_e$. The objective of the domain adversarial learning module is shown as follow
\begin{multline}
\begin{aligned}
\mathcal{L}_{dom}&(\phi_d, \theta_{d,d}, \theta_{d,y})=\\ &\frac{1}{n} \cdot \sum_{x_{i} \in\left(D_{S}, D_{T}\right)} L_{d}\left(C_d\left(H_{d}\left(G(\bm{x});\phi_d\right);\theta_{d,d}\right), d_{i}\right)\\-&\frac{\omega}{n} \cdot \sum_{x_{i} \in\left(D_{S}, D_{T}\right)} L_{E}\left(C_y\left(H_{d}\left(G(\bm{x});\phi_d\right);\theta_{d,y}\right)\right), 
\end{aligned}
\end{multline}
where $\omega$ is the trade-off parameter between the two objectives that shapes the feature during.

\subsection{Model Summary}
By combining the reconstruction and the disentanglement, we summarize the model as follows.

The total loss of the proposed disentangled semantic representation learning for domain adaptation model is formulated as:
\begin{multline}
\begin{aligned}
\mathcal{L}(\phi_y, \theta_{y,d}, \theta_{y,y}, \phi_{d},& \theta_{d,d}, \theta_{d,y}, \theta_r) = \\&\mathcal{L}_{\mathrm{ELOB}}+\beta \mathcal{L} _{sem}+\gamma \mathcal{L}_{dom},
\end{aligned}
\end{multline}
where $\beta$ and $\gamma$ are the hyper-parameters that is not very sensitive and we set $\beta$=1 and $\gamma$=1.

Under the above objective function our model is trained on the source domain using the following procedure 
\begin{multline}
(\hat{\phi_y}, \hat{\theta_{y,y}}, \hat{\phi_d}, \hat{\theta_{d,y}}, \hat{\theta_r}) = \\\mathop{\arg\min}_{\phi_y, \theta_{y,y}, \phi_d, \theta_{d,y}, \theta_r}\mathcal{L}(\phi_y, \theta_{y,d}, \theta_{y,y}, \phi_{d}, \theta_{d,d}, \theta_{d,y}, \theta_r)\\
(\hat{\theta_{y,d}}, \hat{\theta_{d,d}}) = \mathop{\arg\max}_{\theta_{y,d}, \theta_{d,d}}\mathcal{L}(\phi_y, \theta_{y,d}, \theta_{y,y}, \phi_{d}, \theta_{d,d}, \theta_{d,y}, \theta_r).
\end{multline}

The following classifier with the trained optimal parameters is adapted to the target domains. 
\begin{equation}
y = C_y\left(H_{y}\left(G(\bm{x});\hat{\phi_y}\right);\hat{\theta_{y,y}}\right).
\end{equation}
\subsection{Analysis}
In this section, we first show that the error will be reduced under the domain-invariant feature space, and second we further develop an upper bound for the target generalization error.

Following the instruction in \cite{ben2007analysis}. Let $\displaystyle R:\mathcal{X}\rightarrow \mathcal{Z}$ denote the representation function, where $R\triangleq H\circ G$ follows the definition in Section 2. We denote $D_{S}$ as the source distribution over $\displaystyle \mathcal{X}$ and $\tilde{D}_{S}$ as the induced distribution over the feature space $\mathcal{Z}$, i.e., $\displaystyle \operatorname{Pr}_{\tilde{\mathcal{D}}_{S}} [B]\stackrel{\operatorname{def}}{=}\operatorname{Pr}_{\mathcal{D}_{S}}\left[\mathcal{R}^{-1} (B)\right]$ for any measurable event $\displaystyle B$. Similarly, we use parallel notation, $D_{T}$, $\displaystyle \tilde{D}_{T}$ for the target domain. The error on the source domain with a hypothesis $\displaystyle h$ is defined as
\begin{equation}
\begin{aligned}
\epsilon _{S} (h)= & \mathrm{E}_{\mathbf{z} \sim \tilde{D}_{S}}[\mathrm{E}_{y\sim \tilde{f} (\mathbf{z} )}[ y\neq h(\mathbf{z} )]]\\
= & \mathrm{E}_{\mathbf{z} \sim \tilde{D}_{S}} [C(\mathbf{z} )-h(\mathbf{z} )],
\end{aligned}
\label{eq:error}
\end{equation}
where $\displaystyle C:\mathcal{Z}\rightarrow [ 0,1]$ is the label classier defined on $\displaystyle \tilde{D}_{S}$. 
\begin{theorem}
	\label{thm:1}
Assume that the semantic and the domain factors are independent, i.e., $\mathbf{z}_{y} \independent \mathbf{z}_{d}$. Let $\mathbf{z} =\{\mathbf{z}_{y} ,\mathbf{z}_{d}\}$, and the error on the disentangled source and target domain with a hypothesis $\displaystyle h$ is 
\begin{equation}
\begin{aligned}
\epsilon ^{y}_{S} (h)=\epsilon _{S} (h)-\alpha _{S},\\
\epsilon ^{y}_{T} (h)=\epsilon _{T} (h)-\alpha _{T},
\end{aligned}
\end{equation}
where $\displaystyle \mathrm{\alpha _{S}} \coloneqq \mathrm{E}_{\mathbf{z}_{d} \sim \tilde{\mathcal{D}}_{S}}[ C(\mathbf{z}_{d}) -h(\mathbf{z}_{d})]$ and $\mathrm{\epsilon ^{y}_{S} (h)\coloneqq E}_{\mathbf{z}_{y} \sim \tilde{\mathcal{D}}_{S}}[ C(\mathbf{z}_{y}) -h(\mathbf{z}_{y})]$ denotes the error of DSR with respect to $h$ in the source domain, while $\mathrm{\epsilon ^{y}_{T} (h)}$ denotes the error of DSR in target domain.
\end{theorem}
\begin{proof}
Since $\mathbf{z}_{y} \independent \mathbf{z}_{d}$, we can further derive Eq. \ref{eq:error} as follow,
\begin{equation}
\label{eq:thm1}
\begin{aligned}
\epsilon _{S} (h)= & \mathrm{E}_{(\mathbf{z}_{y} ,\mathbf{z}_{d} )\sim \tilde{\mathcal{D}}_{S}} [C(\mathbf{z} )-h(\mathbf{z} )]\\
= & \underbrace{\mathrm{E}_{z_{y} \sim \tilde{\mathcal{D}}_{S_{y}}} [C(\mathbf{z}_{y} )-h(\mathbf{z}_{y} )]+}_{\mathrm{\epsilon ^{y}_{S} (h)}}\\
& \underbrace{\mathrm{E}_{z_{d} \sim \tilde{\mathcal{D}}_{S_{d}}} [C(\mathbf{z}_{d} )-h(\mathbf{z}_{d} )]}_{\mathrm{\alpha _{S}}}\\
= & \mathrm{\epsilon ^{y}_{S} (h)} +\mathrm{\alpha _{S}} .
\end{aligned}
\end{equation}
In the second equality, based on the independence property between $\displaystyle \mathbf{z}_{y}$ and $\displaystyle \mathbf{z}_{d}$, the distribution of $\displaystyle \tilde{\mathcal{D}}_{S}$ can be decomposed into two part so as to the error.

Similarly, we have $\epsilon ^{y}_{T} (h)=\epsilon _{T} (h)-\alpha _{T}$, by decomposing the $\tilde{\mathcal{D}}_{T}$ using the independence property.
\end{proof}
Theorem \ref{thm:1} shows that the disentanglement of the representation space is helpful and might also necessary for obtaining less classify error. Then, in the following Theorem 2, we show that the less classify error on the source domain will tighten the error bound at the target domain.

\begin{theorem}
Let $R(x)$ be a fixed representation function from $\mathcal{X}$ to $\mathcal{Z}$ and $\mathcal{H}$ be a hypothesis space. Let $h^{*} =\arg\min_{h\in \mathcal{H}}( \epsilon _{T} (h),\ \ \epsilon _{S} (h))$, and let $\lambda _{S}$, $\lambda _{T}$ be the errors of $h^{*}$ with respect to $\mathcal{D}_{S}$ and $\mathcal{D}_{T}$ respectively, i.e., $\displaystyle \lambda _{T} \coloneqq \epsilon _{T}\left( h^{*}\right)$ and $\displaystyle \lambda _{S} \coloneqq \epsilon _{S}\left( h^{*}\right)$. We have
\begin{equation}
\epsilon ^{y}_{T} (h)\leq \eta +\epsilon ^{y}_{T} (S)+d_{\mathcal{H}}\left(\tilde{\mathcal{D}}_{S} ,\tilde{\mathcal{D}}_{T}\right),
\end{equation}
where $\eta \coloneqq \epsilon ^{y}_{T} (h^{*} )+\alpha ^{*}_{T} +\epsilon ^{y}_{S}\left( h^{*}\right) +\alpha ^{*}_{S} +\alpha _{S} -\alpha _{T}$.
\end{theorem}
\begin{proof}
Applying \cite[Theorem~1]{ben2007analysis}, we have
\begin{equation}
\label{eq:thm21}
\epsilon _{T} (h)\leq \lambda _{T} +\lambda _{S} +\epsilon _{S} (h)+d_{\mathcal{H}}\left(\tilde{\mathcal{D}}_{S} ,\tilde{\mathcal{D}}_{T}\right) ,
\end{equation}
where 
\begin{equation*}
\displaystyle d_{\mathcal{H}}\left(\tilde{\mathcal{D}}_{S} ,\tilde{\mathcal{D}}_{T}\right) =2\sup _{h\in \mathcal{H}}\left| \operatorname{Pr}_{D_{S}} [\mathcal{Z}_{h} \vartriangle \mathcal{Z}_{h^{*}} ]-\operatorname{Pr}_{D_{T}} [\mathcal{Z}_{h} \vartriangle \mathcal{Z}_{h^{*}} ]\right| .
\end{equation*}
Then based on Theorem 1, combining with Eq.
\ref{eq:thm21} and Eq. \ref{eq:thm1}, we further obtain
\begin{equation}
\begin{aligned}
\epsilon ^{y}_{T} (h) & \leq \lambda _{T} +\lambda _{S} +\epsilon _{S} (h)+d_{\mathcal{H}}\left(\tilde{\mathcal{D}}_{S} ,\tilde{\mathcal{D}}_{T}\right) -\alpha _{T}\\
& \leq \lambda _{T} +\lambda _{S} +\epsilon _{S} (h)+d_{\mathcal{H}}\left(\tilde{\mathcal{D}}_{S} ,\tilde{\mathcal{D}}_{T}\right).
\end{aligned}
\end{equation}
Note that based on Eq. \ref{eq:thm1}, we denote $\lambda _{T} =\epsilon _{T}\left( h^{*}\right) =\epsilon ^{y}_{T}\left( h^{*}\right) +\alpha ^{*}_{T}$.
 and the error upper bound for the target domain can be further derived as follow,
\begin{equation}
\begin{aligned}
\epsilon ^{y}_{T} (h) & \leq \lambda _{T} +\lambda _{S} +\epsilon _{S} (h)+d_{\mathcal{H}}\left(\tilde{\mathcal{D}}_{S} ,\tilde{\mathcal{D}}_{T}\right) -\alpha _{T}\\
& \leq \epsilon ^{y}_{T} (h^{*} )+\alpha ^{*}_{T} +\epsilon ^{y}_{S}\left( h^{*}\right) +\alpha ^{*}_{S} +\epsilon ^{y}_{S} (h)\\
& +\alpha _{S} +d_{\mathcal{H}}\left(\tilde{\mathcal{D}}_{S} ,\tilde{\mathcal{D}}_{T}\right) -\alpha _{T}\\
& \leq \eta +\epsilon ^{y}_{T} (S)+d_{\mathcal{H}}\left(\tilde{\mathcal{D}}_{S} ,\tilde{\mathcal{D}}_{T}\right) ,
\end{aligned}
\end{equation}
where $\eta \coloneqq \epsilon ^{y}_{T} (h^{*} )+\alpha ^{*}_{T} +\epsilon ^{y}_{S}\left( h^{*}\right) +\alpha ^{*}_{S} +\alpha _{S} -\alpha _{T}$.
\end{proof}

\section{Experiments and Results}

\begin{table*}[h]
    \centering
    \resizebox{0.65\textwidth}{23mm}{
    \begin{tabular}{llllllll}
        \hline
        Mehtods                     & $A \rightarrow W$ & $D \rightarrow W$ & $W \rightarrow D$ & $A \rightarrow D$ & $D \rightarrow A$ & $W \rightarrow A$ & Avg\\
        \hline
        ResNet-50 \cite{he2016deep} & 68.4              &96.7               &99.3               &68.9               &62.5               &60.7               & 76.1\\
        TCA \cite{pan2011domain}      & 72.7              &96.7               &99.6               &74.1               &61.7               &60.9               & 77.6\\
        GFK \cite{gong2012geodesic}     & 72.8              &95.0               &98.2               &74.5               &63.4               &61.0               & 77.5\\
        DAN \cite{long2015learning}    & 80.5              &97.1               &99.6               &78.6               &63.6               &62.8               & 80.4\\
        RTN \cite{long2016unsupervised}     & 84.5              &96.8               &99.4               &77.5               &66.2               &64.8               & 81.6\\
        DANN \cite{ganin2016domain}   & 82.0              &96.9               &99.1               &79.7               &68.2               &67.4               & 82.2\\
        ADDA \cite{tzeng2017adversarial}   & 86.2              &96.2               &98.4               &77.8               &69.5               &68.9               & 82.9\\
        JAN \cite{long2017deep}     & 85.4              &97.4               &99.8               &84.7               &68.6               &70.0               & 84.3\\
        MSTN \cite{xie2018learning}     & 86.9              &96.7               &99.9               &87.3               &66.9               &68.4               & 84.3\\
        CDAN-M \cite{long2018conditional}    & \textbf{93.1}   &98.6   &\textbf{100.0}     &\textbf{93.4}      &71.0               &70.3               & 87.7\\
        DSR\_DM($\delta=1$)            & 90.7              &97.1               &99.2               &88.8               &71.5               &71.2               & 86.4\\
        DSR\_DM($\delta=2$)            & 91.5              &97.2               &99.5               &88.5               &71.5               &72.1               & 86.7\\
        Ours                        & \textbf{93.1}     &\textbf{98.7}      &99.8               &92.4               &\textbf{73.5}      &\textbf{73.9}      & \textbf{88.6}\\
        \hline
    \end{tabular}}
    \caption{Accuracy (\%) on Office-31 for unsupervised domain adaptation (ResNet)}
    \label{tab:office_31}
\end{table*}

\begin{table*}[h]
    \centering
    \resizebox{0.98\textwidth}{18mm}{
        \begin{tabular}{llllllllllllll}
            \hline
            Method&Ar$\rightarrow$Cl&Ar$\rightarrow$Pr&Ar$\rightarrow$Rw&Cl$\rightarrow$Ar&Cl$\rightarrow$Pr&Cl$\rightarrow$Rw&Pr$\rightarrow$Ar&Pr$\rightarrow$Cl&Pr$\rightarrow$Rw&Rw$\rightarrow$Ar&Rw$\rightarrow$Cl&Rw$\rightarrow$Pr&Avg \\
            \hline
            ResNet-50 \cite{he2016deep}           & 34.9 & 50.0 & 58.0 & 37.4 & 41.9 & 46.2 & 38.5 & 31.2 & 60.4 & 53.9 & 41.2 & 59.9 & 46.1 \\ 
            DAN \cite{long2015learning}               & 43.6 & 57.0 & 67.9 & 45.8 & 56.5 & 60.4 & 44.0 & 43.6 & 67.7 & 63.1 & 51.5 & 74.3 & 56.3 \\
            DANN \cite{ganin2016domain}             & 45.6 & 59.3 & 70.1 & 47.0 & 58.5 & 60.9 & 46.1 & 43.7 & 68.5 & 63.2 & 51.8 & 76.8 & 57.6 \\
            JAN \cite{long2017deep}                  & 45.9 & 61.2 & 68.9 & 50.4 & 59.7 & 61.0 & 45.8 & 43.4 & 70.3 & 63.9 & 52.4 & 76.8 & 58.3 \\
            MSTN \cite{xie2018learning}               & 49.3 & 67.6 & 74.7 & 49.6 & 63.7 & 64.5 & 50.7 & 43.7 & 73.0 & 62.7 & 52.2 & 77.9 & 60.9 \\
            CDAN-M  \cite{long2018conditional}            & 50.6 & 65.9 & 73.4 & 55.7 & 62.7 & 64.2 & 51.8 & 49.1 & 74.5 & \textbf{68.2} & \textbf{56.9} & \textbf{80.7} & 62.8 \\ 
            DSR\_DM($\delta=1$)            & 50.5 & 69.2 & 75.5 & 52.9 & 64.0 & 65.6 & 53.0 & 45.0 & 73.5 & 63.2 & 50.8 & 78.3 & 61.8 \\
            DSR\_DM($\delta=2$)               & 52.3 & 70.9 & 76.5 & 54.0 & 65.2 & 67.0 & 53.5 & 45.0 & 73.7 & 62.7 & 50.7 & 78.7 & 62.5 \\
            DSR                                  & \textbf{53.4} & \textbf{71.6} & \textbf{77.4} & \textbf{57.1} & \textbf{66.8} & \textbf{69.3} & \textbf{56.7} & \textbf{49.2} & \textbf{75.7} & 68.0 & 54.0 & 79.5 & \textbf{64.9} \\  
            \hline
    \end{tabular}}
    \caption{Accuracy (\%) on Office-home for unsupervised domain adaptation (ResNet)}
    \label{tab:office_home}
\end{table*}

\subsection{Setup}
Office-31 is a standard benchmark for visual domain adaptation, which contains 4,652 images and 31 categories from three distinct domains: Amazon (A), Webcam (W) and DSLR (D).

Office-Home is a more challenging domain adaptation dataset than Office-31, which consists of around 15,500 images from 65 categories of everyday objects. This dataset is organized into four domains: Art (Ar), Clipart (Cl), Product (Pr) and Real-world (Rw).

\subsubsection{Compared Approaches }
Beside the classical approaches, we also compare our disentangled semantic representation model with some deep transfer learning methods.
Two recently proposed semantic enhanced methods, CDAN \cite{long2018conditional} and MSTN \cite{xie2018learning}, are also compared in the experiment. Note that, CDAN conditions the adversarial adaptation models on discriminative information conveyed in the classifier predictions, and MSTN learns semantic representation by aligning labeled source centroid and pseudo-labeled target centroid.

\subsection{Result}
\subsubsection{Office-31 Result} 
The classification accuracies on the Office-31 dataset for unsupervised domain adaptation on ResNet-50 are shown in Table \ref{tab:office_31}. Our DSR model significantly outperforms all other baselines on most of the transfer tasks. It is remarkable that our method promotes the classification accuracies substantially on hard transfer tasks, e.g., \text{D$\rightarrow$A, W$\rightarrow$A}, and produces comparable results on the other relatively simple tasks, e.g.,  \text{A$\rightarrow$W, D$\rightarrow$W}. However, the result of DSR on the \text{W$\rightarrow$D and A$\rightarrow$D} tasks are lower than that of some compared approaches. This is because domain $\textsl{DSLR}$ (D) only has total 498 images for 31 classes and some classes even only have less than 10 samples, it's insufficient for our method to reconstruct the disentangled semantic representation. We conduct the Wilcoxon signed-rank test \cite{lowry2014concepts} on the reported accuracies, the results are as follows: on the Office-31 data set, our method is comparable with CDAN-M, and significantly outperforms all the other baselines with the p-value threshold 0.05.  

\subsubsection{Office-home Result} 
Similar to the results on Office-31, the DSR model also outperforms all other baselines on most of the tasks, as reported in Table 2. Note that our method achieves a great improvement when the source domain is Art (Ar) or Clipart (Cl) and performs slightly worse than CDAN-M when the source domain is Real World (RW). This is because the Ar and the clipart (Cl) domain contain relatively simpler pictures and more complex scenarios than the other domains and our DSR model can extract such semantic representation easily and thus achieve good performance on this source domain. However, in the Real World (RW) domain, the pictures are taken in real life and there a lot of ambiguous samples, .e.g., pictures with monitor, computer and laptop are tagged with the same label, which implies that the semantic information is difficult to be disentangled and extracted on this domain. We also conduct the Wilcoxon signed-rank test \cite{lowry2014concepts} on the reported accuracies, our method significantly outperforms the baselines, with the p-value threshold 0.05.

\subsubsection{The Study of the Disentangled Semantic Representation} 
To study the effectiveness of the disentangled semantic representation, we compare our methods with two approaches using similarly adversarial learning strategy but different representations on the task \textbf{Ar}$\rightarrow$\textbf{Cl}. Figure. \ref{fig:tsne} (a)-(c) show the visualization of the extracted features using  t-SNE. As shown in the figure, DSR obtains the best alignment among the three representations. Both DANN and MSTN have a large number of samples are falsely aligned. This result verifies the effectiveness of DSR.

Such results also can be observed in Table 1 and Table 2. For example, DSR achieves remarkably outstanding results on some transfer tasks, e.g. \text{D$\rightarrow$A, W$\rightarrow$A} on the Office-31 dataset and all the tasks except for those using RW as the source domain on the Office-home dataset. These results show a common phenomenon that the samples of the source domain are more complex than that of the target domain, i.e., the source domain has more scenarios than the target domain. Such common phenomenon also shows the advantage of our disentangled semantic representation.

\begin{figure}[H] 
	\centering
	\includegraphics[width=0.48\textwidth]{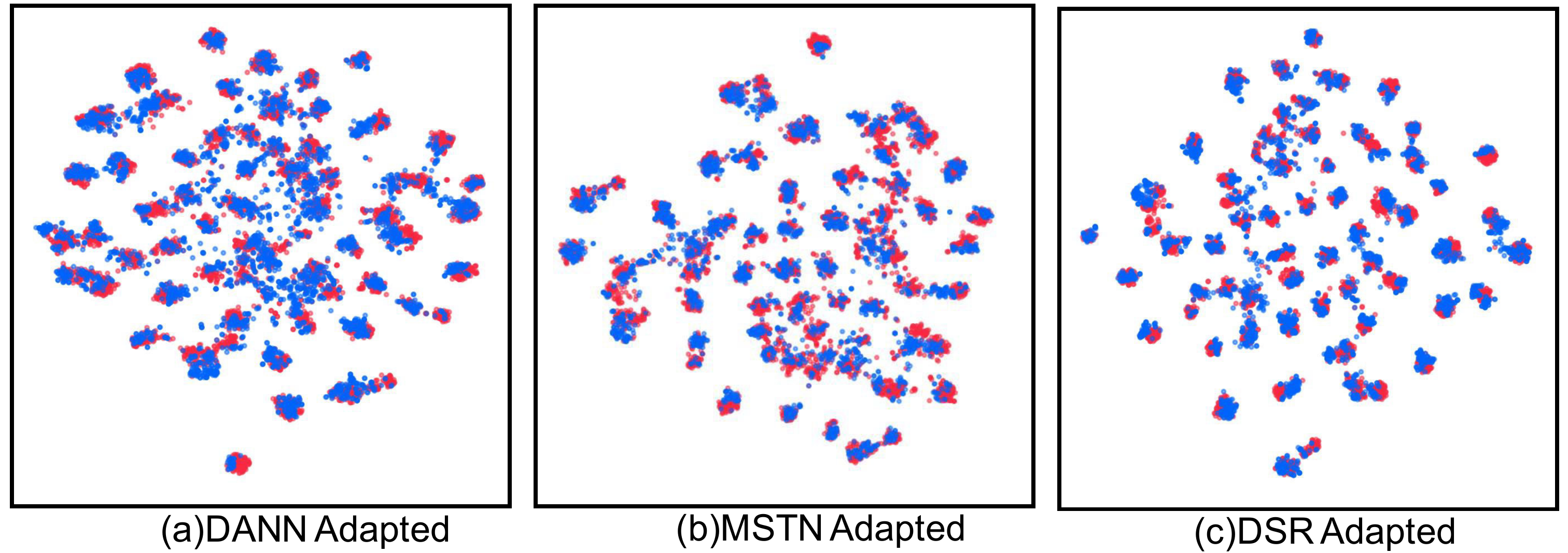}
	\caption{The t-SNE visualization of deep features extracted by DANN (a), MSTN (b) and DSR (c). The red points are source domain samples and the blue points are target domain samples.}
	\label{fig:tsne}
\end{figure}

\subsubsection{Ablation Study of the Dual Adversarial Learning}
To study the effectiveness of the dual adversarial learning module, we first train the standard DSR model until convergence, then train the model without the domain adversarial learning module. Such ablated model is named DSR\_WD in the experiment. The experiment results are shown in Table \ref{tab:office_31} and Table \ref{tab:office_home}.
Comparing the result of DSR and DSR\_WD ($\delta$=1), we find that the performance drops because the semantic information is drained away without disentanglement. Even when we use $\delta=2$, the performance of the ablated model is still worse than the original one. These results verify that the dual adversarial learning module can push the semantic information into the semantic latent variables $\mathbf{z}_y$, and simultaneously, push the domain information into the domain latent variables $\mathbf{z}_d$.

\section{Conclusion}
This paper presents a disentangled semantic representation model for the unsupervised domain adaptation task. Different from previous work
 , our approach extracts the disentangled semantic representation on the recovered latent space, following the causal model of the data generation process. Our approach is also featured with the variational auto-encoder based latent space recovery and the dual adversarial learning based disentangle of the representation. The success of the proposed approach not only provides an effective solution for the domain adaptation task, but also opens the possibility of disentanglement based learning methods.
\section*{Acknowledgments}
This research was supported in part by NSFC-Guangdong Joint Found (U1501254), Natural Science Foundation of China (61876043), Natural Science Foundation of Guangdong (2014A030306004, 2014A030308008), Guangdong High-level Personnel of Special Support Program (2015TQ01X140) and Pearl River S\&T Nova Program of Guangzhou (201610010101). Kun Zhang would like to acknowledge the support by National Institutes of Health (NIH) under Contract No. NIH-1R01EB022858-01, FAINR01EB022858, NIH-1R01LM012087, NIH5-5U54HG008540-02, and FAIN- U54HG008540, by the United States Air Force under Contract No. FA8650-17-C-7715, and by National Science Foundation (NSF) EAGER Grant No. IIS-1829681. The NIH, the U.S. Air Force, and the NSF are not responsible for the views reported here. We would like to thank Dr Tom Fu from ADSC and professor Ke Yiping from Nanyang Technological University for their help and supports on this work.

\appendix

\bibliographystyle{named}
\bibliography{ijcai19}

\end{document}